\documentclass[11pt,a4paper]{article}

\usepackage{graphicx,amssymb,amsmath, amsthm, amscd,xcolor}
\usepackage{amsthm}
\usepackage{csquotes}
\usepackage{nicefrac}

\newtheorem{theorem}{Theorem}
\newtheorem{lemma}[theorem]{Lemma}

\newtheorem{definition}[theorem]{Definition}

\newcommand{\RR}{\mathbb{R}}
\newcommand{\RRone}{\RR_{\ge 1}}
\newcommand{\RROne}{\RR_{> 1}}

\newcommand{\oldnew}[2]{#2}

\begin{document}
\title{Lower Bounds for Searching Robots, some Faulty\thanks{This research was initiated at the 15$^\mathrm{th}$ Gremo's Workshop on  Open Problems (GWOP), Pochtenalp, Switzerland, June 12-16, 2017.}}
\author{Andrey Kupavskii\thanks{The research was partially supported by the Swiss National Science Foundation grants no. 200020-162884 and 200021-175977 and  by the EPSRC grant no. EP/N019504/1.}\\
University of Birmingham and \\
Moscow Inst.\ of Physics and Technology \\
{\tt kupavskii@yandex.ru}
\and
Emo Welzl\\
Department of Computer Science\\
ETH Zurich\\
{\tt emo@inf.ethz.ch}
}

\maketitle
\vspace{-3.5cm}

\vspace{1.3cm}

\vspace{2.7cm}

\begin{abstract}

Suppose we are sending out $k$ robots from $0$ to search the real line at constant speed (with turns) to find a target at an unknown location; $f$ of the robots are faulty, meaning that they fail to report the target although visiting its location (called crash type). The goal is to find the target in time at most $\lambda |d|$, if the target is located at $d$, $|d| \ge 1$, for $\lambda$ as small as possible. We show that this cannot be achieved for
$$\lambda < 2\frac{\rho^\rho}{(\rho-1)^{\rho-1}}+1,~~ \rho := \frac{2(f+1)}{k}~, $$
which is tight due to earlier work \oldnew{}{(see J. Czyzowitz, E. Kranakis, D. Krizanc, L. Narayanan, J. Opatrny, PODC'16, where this problem was introduced)}.
This also gives some better than previously known lower bounds for so-called Byzantine-type faulty robots that may actually wrongly report a target.

In the second part of the paper we deal with the $m$-rays generalization of the problem, where the hidden target is to be detected on $m$ rays all emanating at the same point. Using a generalization of our methods, along with a useful relaxation of the original problem, we establish a tight lower for this setting as well (as above, with $\rho := \nicefrac{m(f+1)}{k}$).
When specialized to the case $f=0$, this resolves the question on parallel search on $m$ rays, posed by three groups of scientists some 15 to 30 years ago: by Baeza-Yates, Culberson, and Rawlins; by Kao, Ma, Sipser, and Yin; and by Bernstein, Finkelstein, and Zilberstein. The $m$-rays generalization is known to have connections to other, seemingly unrelated, problems, including hybrid algorithms for on-line problems, and so-called contract algorithms.

\end{abstract}
\vspace{1cm}

\newpage
\clearpage
\pagenumbering{arabic}
\setcounter{page}{1}

\section{Introduction}

The following problem was raised in 1963 by Bellman \cite{Bell} and addressed in the 60s by Beck \cite{Be1, Be2, Be3} and Franck \cite{Franck} (quoting from Beck's paper \cite{Be1}):
\smallskip

\begin{displayquote}
\noindent
{\sl ``A man in an automobile searches for another man who is located at some point of a certain road. He starts at a given point and knows in advance the probability that the second man is at any given point of the road. Since the man being sought might be in either direction from the starting point, the searcher will, in general, have to turn around many times before finding his
target. How does he search so as to minimize the expected distance traveled? When can this minimum expectation actually be achieved?''}
\end{displayquote}
\smallskip

The authors were interested in proving the existence of and finding  optimal strategies for different distributions, as well as proving non-existence results. In particular, in \cite{Be3} it is shown that, without apriori knowledge of the distribution, we cannot hope the distance traveled to be smaller than 9 times the expected distance to the target.

A variant of the problem was rediscovered in the late 80s by computer scientists, and it became known as the cow path problem (or cow at a fence problem). Quoting from \cite{BYC}:
\smallskip

\begin{displayquote}
{\sl ``A cow comes to an infinitely long straight fence. The cow knows that there is a gate in the fence, and she wants to get to the other side. Unfortunately, she doesn't know where the gate is located. Assume that the gate is positioned in an integer number of steps away from the cow and that the cow can only recognize the gate when directly upon it. How can she optimally find the gate?''}
\end{displayquote}
\smallskip

By optimality the authors mean the smallest worst-case ratio between the distance traveled and the distance to the gate (the {\it competitive ratio}). It is easy to see that the strategy in which the cow goes $1$ to the left, then back and $2$ to the right, then back and $4$ to the left etc. gives the worst-case ratio of 9. It is shown in \cite{BYC} (de-facto reproving one of the results of \cite{Be3}) that one cannot hope for a better ratio.

The \oldnew{cow at a fence}{cow path} problem gave rise to a significant line of research in computer science. Several variants of the problem were addressed (\cite{BoCa, DeFe}), as well as its planar variants \cite{BYC, BYC2} that go back to another famous search problem, the lost at sea problem \cite{Bell2, Isbe}. We refer the reader to the book \cite{AlGal}, in which many of the single robot search problems, including the \oldnew{cow at a fence}{cow path} problem, are discussed. Motivated by numerous applications, different possible variants of collective search problems were considered in the context of random walks \cite{Aletal}, forging \cite{FeiKo}, unmanned aerial vehicles \cite{Polycarp}.


In \cite{BR}, \cite{FR} the authors proposed the following parallel search version of the \oldnew{cow at a fence}{cow path} problem:
We are sending out $k$ robots from $0$ to explore the real line $\RR$, each one traveling at unit speed, changing directions according to some strategy we are free to prescribe.  There is a target hidden at a point $x$ of the real line, where $|x|\ge 1$.
The robots may detect the target only when they arrive at the point where the target is hidden. Moreover, $f$ robots are {\it faulty}. The problem is to determine the minimal {\it competitive ratio}: the supremum, over all possible target locations $x$, $|x| \ge 1$, of the ratio $\frac{\tau(x)}{|x|}$, for $\tau(x)$ the time taken for the non-faulty robots to be sure about the location of the target at $x$.

In \cite{FR} the authors focus on the faults of the {\it crash} type: When passing through the target, the robot does not detect it. We denote the competitive ratio in this setting by $A(k,f)$.

In \cite{BR} the authors study the faults of the {\it Byzantine} type: A faulty robot may stay silent even when it detects or visits the target, or may claim that it has found the
target when, in fact, it has not found it. In this case we denote the competitive ratio by $B(k,f)$.


One of our contributions is the complete resolution of the case of faults  of crash type (or {\it silent} robots):
\begin{theorem}\label{thm1} Let us denote $s:=2(f+1)-k$ and $\rho := \frac{2(f+1)}{k}$. If $0<s\le k$ (i.e.\ $1 < \rho \le 2$) then we have
\begin{equation}\label{eq01}
A(k,f)=2\sqrt[k]{\frac{(k+s)^{k+s}}{s^sk^k}}+1 = 2\frac{\rho^{\rho}}{(\rho-1)^{\rho-1}}+1~.\end{equation}
\end{theorem}

We note that if $s\le 0$, that is, $k\ge 2(f+1)$, then by sending $f+1$ of the robots to $\infty$ and $f+1$ of the robots to $-\infty$ we achieve a competitive ratio 1. We also note that $s>k$, i.e.\ $f+1 > k$, means that $k=f$, in which case we cannot find the target since all robots are faulty.

In \cite{FR} the authors exhibited a natural strategy which achieves the ratio in \eqref{eq01}, as well as provided some lower bounds for $A(k,f)$. Thus, our contribution in Theorem~\ref{thm1} is the matching lower bound.

A lower bound for crash-type faulty robots is also a lower bound for Byzantine-type faulty robots. By this relation, our bounds also improve on the known bounds for Byzantine-type faulty robots, e.g.\ the bound of $B(3,1) \ge 3.93$  from \cite{BR} is now at $B(3,1) \ge \frac{8}{3}\sqrt[3]{4}+1 \approx 5.23$.\vskip+0.1cm

More importantly, we develop a method which allows to deal with problems of this type. We demonstrate this in the proof of Theorem~\ref{thm1} in Section~\ref{sec2}. In Section~\ref{sec3}, we generalize Theorem~\ref{thm1} to the case of $m$ rays, thus resolving an old question, which appeared in at least three different papers. We also give several useful relaxations of the original problem, which shed more light on this group of questions.
Some of the less important proofs are deferred to Appendix.

\section{Proof of Theorem \ref{thm1}}\label{sec2}
We are given $k$ robots, $f$ of them faulty of crash type. To assure that the target at $x \in \RR$ is found in time at most $\lambda |x|$, the point $x$ has to be visited by at least $f+1$ robots in time (otherwise the adversary will place the target there and choose the first $f$ robots arriving at $x$ to be faulty and stay silent). Therefore, for each pair of points $(x,-x)$, $x \in \RRone$, at least $2(f+1)-k=s$ robots need to pass through both points before time $\lambda x$ elapses. If a robot visits both $(x,-x)$ in time $\lambda x$, then we say that it {\it $\lambda$-covers} $x$. Similarly, if $x$ is $\lambda$-covered $s$ times by a group of several robots, we say that it is {\it $s$-fold $\lambda$-covered}. Since we are going to work with coverings in different settings, let us give the following general definition. 
\begin{definition}{\rm Fix $\lambda>1$ and $s\in \mathbf N$. For a covering problem of a certain object with some set of rules (called {\it the  setting}), and a certain set $S$ we say that $S$ is {\it $\lambda$-covered}, if for any point in $S$ at distance $x$ from the origin it was visited by the robot within time $\lambda x$ (where ``visited'' is interpreted within the setting, and should comply with the rules of the setting). We may also say that it is the {\it $\lambda$-covering of $S$}. Similarly, we say that a robot or set of robots we say that it produces an $s$-fold $\lambda$-covering of $S$ (in some setting), if each point was $\lambda$-covered at least $s$ times.}\end{definition}\vskip-0.2cm
Note that this definition is strategy-dependent. Let us now formalize the covering setting that we are working with in this proof.\vskip+0.1cm

\textsc{The symmetric line-cover setting:} The goal is to cover $\RRone$. The robot (robots) moves on $\RR$ with unit speed, starting from the origin. A point $x \in \RRone$ is covered by a robot $r$ at the moment when $r$ visited both $x$ and $-x$, it is \emph{$\lambda$-covered}, if that happens within time $\lambda x$. A robot can cover any point at most once. \vskip+0.1cm

Throughout the proof we will be working with the $s$-fold $\lambda$-covering of $\RRone$ with $k$ robots in the symmetric line-cover setting for fixed $s$ and $\lambda$.  Therefore, slightly abusing notation, we will use the term \emph{$\pm$-covering} to refer to this covering problem.\vskip+0.1cm


By the above discussion, any valid strategy with competitive ratio $\lambda$ for the original problem (of detecting a target with $k$ robots, $f$ of which are faulty of the crash type) is also a strategy for $\pm$-covering $\RRone$ with $k$ robots. We actually obtain a lower bound as in \eqref{eq01} for the competitive ratio for the $\pm$-covering strategies. Our goal is to prove the following result.

\begin{theorem}
\label{th:PlusMinusCover}
It is impossible to $\pm$-cover $\RRone$ with $k$ robots if
$$\lambda<2\sqrt[k]{\frac{(k+s)^{k+s}}{s^sk^k}}+1 ~.$$
\end{theorem}

For each individual robot $r$ we will restrict ourselves to strategies described by an infinite nondecreasing sequence $T= T^{(r)} = (t_1,t_2, t_3, \ldots)$ over $\RRone$, with the interpretation that the robot is sent till $+t_1$ in the positive direction, till $-t_2$ in the negative direction, till $+t_3$ in the positive direction, etc. We argue that this carries no loss of generality.

First, it is clear that for $\pm$-covering, we may assume to start in the positive direction. Second, if the robot ever turns in territory that was visited before by this robot, we can shift the turning point in the direction where we came from (since we skip only parts of $\RR$ which we have visited already, and any further visits occur now even earlier). These two observations restrict already to movements that alternate between turning at positive and negative numbers, actually with the positive turning points increasing and the absolute values of negative turning points increasing. Now suppose for $0 < x_2 < x_1 $, the robot turns at $x_1$ and then at $-x_2$.  Note that the interval $(x_2,x_1]$ is covered, but $[-x_1,x_2)$ is not, so as for $\pm$-covering, we may as well turn at $x_2$ instead of $x_1$ (since after turning at $-x_2$, we will return to $(x_2,x_1]$ before we actually extend our trip to points smaller than $-x_2$). Thus, we can transform any strategy to an at least as efficient strategy of the type as above.

Given a strategy $T=T^{(r)}=(t_1,t_2, t_3, \ldots)$ of a single robot $r$, it is easy to see that for $x$ with $ t_{i-1}<x\le t_i$, it takes exactly $2(t_1+t_2+\cdots+t_i)+x$ time to visit both $x$ and $-x$ (note that indeed the sum goes up to ``$t_i$''). So in order for the robot to $\lambda$-cover $x$, we need
\begin{eqnarray}
x \ge \frac{1}{\mu}(t_1+t_2+\cdots+t_i) \mbox{~~~for $\mu := \frac{\lambda - 1}{2}$.}
\end{eqnarray}
Let us set
\begin{eqnarray}
\label{eq:BeginVsEnd}
t_i'' := \max\left\{\frac{1}{\mu} (t_1 + t_2 + \cdots + t_i), t_{i-1}\right\}
\end{eqnarray}
unless this value exceeds $t_i$, when we leave $t_i''$ undefined. All $i$ with $t_i'' $ defined we call \emph{fruitful}.
It is clear that robot $r$  $\lambda$-covers exactly
$$
\mathrm{Cov}_\mu(T) := \bigcup_{i ~\text{fruitful}} [t_i'' ,t_i] ~.
$$
Note that if we skip a turn $t_{i^*}$ in $T$, then, of course, we lose the interval $[t_{i^*}'' ,t_{i^*}]$, but following intervals (for $i > i^*$) extend even further to the left. That is, turning points that are not fruitful can be skipped, in this way definitely $\lambda$-covering at least as much. Moreover, it becomes clear that if $t_{i+1} = t_{i}$ (hence $t_{i+1}'' = t_{i+1}$, if defined), then we can skip $t_{i+1}$ in the strategy thereby covering at least as much.
\medskip

We move now to considering all robots, not just individual ones.
If we are given a $\pm$-covering by $k$ robots,
then each point of $\RRone$ is contained in  $\mathrm{Cov}_\mu(T^{(r)})$ for at least $s$ (out of the $k$) robots. By truncating some of the intervals from $[t_i'',t_i]$ to half-open intervals  $(t_i',t_i]$ (with $t_i'' \le t_i' < t_i$) or even skipping some intervals we may assume that each point of $\RROne$ is contained in {\it exactly} $s$ such assigned intervals (of $s$ different robots), and that the turning points of each robot coincide with the right ends of the corresponding intervals. Recall here, that if an interval is not needed, we can actually skip the corresponding turning point in the robot's strategy without affecting the validity of the other intervals used. We call the new half-open intervals $(t_i',t_i]$ {\it assigned intervals}. Let us emphasize, as a consequence of $ t_i' \ge t_i''$ and (\ref{eq:BeginVsEnd}), that
\begin{eqnarray}
\label{eq:ShiftedBeginVsEnd}
t_i' \ge \frac{1}{\mu} (t_1 + t_2 + \cdots + t_i)
\end{eqnarray}
or, equivalently,
\begin{eqnarray}
\label{eq:EndVsShiftedBegin}
t_i \le \mu t_i' -  (t_1 + t_2 + \cdots + t_{i-1}) ~.
\end{eqnarray}

Next we accumulate all assigned intervals of all $k$ robots in one sequence, sorted by their left endpoints, ties broken arbitrarily. Consider a prefix ${\cal P}$ of this sequence, long enough so that all robots have already some interval present in ${\cal P}$. What part of $\RROne$ do these intervals cover, and how often? There is a value $a = a({\cal P}) \ge 1$ such that $(1,a]$ is covered $s$ times, and then there are points $$a=:a_s\le a_{s-1}\le a_{s-2}\le\ldots\le a_1 ~,$$ such that every point in $(a_{j+1},a_j]$ is covered exactly $j$ times, $j=1,\ldots, s-1$, and $(a_1,\infty)$ is not yet covered. (Note that some of these intervals may be empty.) Associate with ${\cal P}$ the \emph{multiset} $$A({\cal P}) := \{ a_s, a_{s-1}, \ldots, a_1\}$$ (which can be seen as a description of the \emph{covering situation} of ${\cal P}$).

We need one more notion. Define the {\it load}, $L^{(r)}({\cal P})$, of robot $r$ in ${\cal P}$ as the sum of all turning points occuring in the intervals of $r$ in ${\cal P}$, i.e.\
$$
L^{(r)}({\cal P}) : = t_1^{(r)} + t_2^{(r)} + \cdots + t_{i_r}^{(r)} \mbox{~~ for $i_r$ s.t. $\left\{\begin{array}{ll} \mbox{$({t'_{i_r}}^{\!\!\!(r)},t_{i_r}^{(r)}]$ in ${\cal P}$, and} \\ \mbox{$({t'_{i_r+1}}^{\!\!\!\!\!\!\!\!\!\!(r)\,\,},t_{i_r+1}^{(r)}]$ not in ${\cal P}$.} \end{array}\right.$}
$$  We immediately observe
\begin{equation}
\label{eq:LoadBound}
L^{(r)}({\cal P}) \le \mu {t'_{i_r}}^{\!\!\!(r)} \le \mu a \mbox{~~ (due to (\ref{eq:ShiftedBeginVsEnd})).}
\end{equation}
As we extend the prefix ${\cal P}$ by the next assigned interval, this interval has to be of the form $({t'_i}^{(r^*)},t_i^{(r^*)}]$, for some robot $r^*$. Note
$$
{t'_i}^{(r^*)}  = a \mbox{~~and~~} t_i^{(r^*)} = \mu^* a - L^{(r^*)}({\cal P}) \mbox{~, for some $\mu^* \le \mu$ (due to (\ref{eq:EndVsShiftedBegin})).}
$$
Set ${\cal P}^+$ to be the prefix with ${\cal P}$ extended by this next interval $({t'_i}^{(r^*)},t_i^{(r^*)}]$. We immediately see that loads of robots do not change as we move from ${\cal P}$ to ${\cal P}^+$, except for robot $r^*$ for which we get
$$
L^{(r^*)}({\cal P}^+) = L^{(r^*)}({\cal P}) + t_i^{(r^*)} = \mu^* a ~.
$$
Also $a_s$ in $A({\cal P})$ gets replaced by $t_i^{(r^*)} = \mu^* a - L^{(r^*)}({\cal P})$ in $A({\cal P}^+)$, all other elements stay the same. (Note that still, $a({\cal P}^+) = a({\cal P})$ is a possibility if $a_s = a_{s-1}$.)

We want to analyze the function
\begin{equation}
f({\cal P}):=\prod_{r=1}^k \left[\frac{\left(L^{(r)}({\cal P})\right)^s}{\prod_{y \in A({\cal P})}y}\right]
\end{equation}
and understand how it changes as we move from ${\cal P}$ to ${\cal P}^+$. With $L^{(r)}({\cal P}) \le \mu a$ for all robots $r$ and $y \ge a$ for all $y \in A({\cal P})$ we get
\begin{equation}
f({\cal P}) \le \left( \frac{(\mu a)^s}{a^s}\right)^k = \mu^{ks} ~,
\label{eq:fBound}
\end{equation}
i.e.\ $f$ stays bounded. With the changes of $A$ and loads from ${\cal P}$ to ${\cal P}^+$ described, we see that
$$
\frac{f({\cal P}^+)}{f({\cal P})} = \frac{a^k}{\left(L^{(r^*)}({\cal P})\right)^s} \cdot \frac{(\mu^* a)^s}{\left(\mu^* a - L^{(r^*)}({\cal P})\right)^k} = \frac{{\mu^*}^s}{x^s (\mu^* - x)^k}
$$
for $x := \frac{L^{(r^*)}({\cal P})}{a}$, $0 < x < \mu^*$. We show that this ratio is larger than 1 if $\mu^* \le \mu$ is too small, independent of $x$. We use the following lemma.
\begin{lemma}\label{lemmax}
For $\mu^*>0$ the polynomial $x^s(\mu^*-x)^k$ maximizes for $x = \frac{s\mu^*}{k+s}$ in the range $x \in \RR$, $0 < x < \mu^*$.
\end{lemma}
\begin{proof}
The derivative of the polynomial is $(\mu^*-x)^{k-1}x^{s-1}(s(\mu^*-x)-kx)$, which has only one zero $x=\frac{s\mu^*}{k+s}$ in  $(0,\mu^*)$. Since the values at $0$ and $\mu^*$ of the polynomial are 0, the zero of the derivative corresponds to the maximum of the function.
\end{proof}
\begin{lemma}\label{lemgeom} For $0 < x < \mu^*$,
$
\frac{{\mu^*}^s}{x^s(\mu^* - x)^k} \ge \frac{(k+s)^{k+s}}{s^sk^k{\mu^*}^{k}}
$
and thus
$
\frac{{\mu^*}^s}{x^s(\mu^* - x)^k} \ge \delta
$
for $\delta := \frac{(k+s)^{k+s}}{s^sk^k\mu^{k}} > 1$, provided $\mu < \sqrt[k]{\frac{(k+s)^{k+s}}{s^sk^k}}$.
\end{lemma}
\begin{proof}
The first inequality follows immediately from Lemma~\ref{lemmax} after substituting the value $x=\frac {s\mu^*}{k+s}$. The second inequality is obvious since we know that $\mu^*\le \mu$.
\end{proof}

Applying Lemma~\ref{lemgeom}, we get that  $\frac{f({\cal P}^+)}{f({\cal P})}\ge \delta>1$, provided $\mu < \sqrt[k]{\frac{(k+s)^{k+s}}{s^sk^k}}$. It implies that $f({\cal P})$ is unbounded for larger and larger prefixes ${\cal P}$. This contradicts the bound on  $f({\cal P})$ from (\ref{eq:fBound}). Therefore, to have a valid strategy, one must have $\mu \ge \sqrt[k]{\frac{(k+s)^{k+s}}{s^sk^k}}$, which concludes the proof of the Theorem~\ref{th:PlusMinusCover} and thus of Theorem~\ref{thm1}.

\section{Generalization to $m$ Rays}\label{sec3}

The following natural generalization of the original cow-path problem was studied by several authors \cite{BYC2}, \cite{BFZ},  \cite{KMS}, \cite{KRT}, \cite{Sch}.  Consider  $m$ rays emanating from 0, and assume that there is a hidden target on one of the rays. We  send a robot from $0$ at unit speed, and the goal is to find (pass over) the target. It was shown in \cite{BYC} that the best possible competitive ratio for the problem is $1+2\frac{m^m}{(m-1)^{m-1}}$, and this is tight.

This problem got a lot of attention due to its numerous connections to other, seemingly unrelated, problems. In particular, it is related to hybrid algorithms \cite{ABM}, \cite{FRR}, \cite{KMS} for on-line problems. We quote from \cite{KMS} (with ``$m$'' for ``$w$'' and ``$k$'' for ``$\lambda$''):
\begin{displayquote}
{\sl ``We study on-line strategies for solving problems with hybrid algorithms.
There is a problem $Q$ and $m$ \emph{basic} algorithms for solving $Q$. For some
$k \le m$, we have a computer with $k$ disjoint memory areas, each of which
can be used to run a basic algorithm and store its intermediate results. In
the worst case, only one basic algorithm can solve $Q$ in finite time, and all
of the other basic algorithms run forever without solving $Q$. To solve $Q$
with a \emph{hybrid} algorithm constructed from the basic algorithms, we run a
basic algorithm for some time, then switch to another, and continue this
process until $Q$ is solved. The goal is to solve Q in the least amount of
time.''}
\end{displayquote}


Interpret the calculations done while performing the $i$-th basic algorithm as the $i$-th ray emanating at $0$ (where $0$ corresponds to the initial state). Then, being at a point $x_i$ in the calculations of the $i$-th algorithm, it costs us at most $x_i+x_j$ to pass from this point to the point $x_j$ in the calculations of the $j$-th algorithm. This interpretation was given in \cite{FRR}, using which the authors exhibited an algorithm, which is competitive with respect to each of a given set of algorithms.


Another related field is contract algorithms. The connection was established in \cite{BFZ}, and the setting is as follows. A processor is supposed to advance in $m$ computational problems, until it is interrupted. Upon interruption, it is given the index $i\in \{1,\ldots, m\}$, and is returns its most advanced computation on the $i$-th problem. Again, interpreting each problem as a ray emanating from $0$, we easily see the connection to the $m$-path problem.

Both papers \cite{BFZ} and \cite{KMS} were concerned with the variation of the $m$-path problem, in which $k$ robots are conducting the search simultaneously. When introducing several robots, we may use two possible measures for competitive ratio: $\frac Td$ and $\frac Dd$, where $T$ is the time spent until the target was found (given that all robots travel at unit speed), and $D$ is the total distance travelled by all robots until the target was found. The distance version of the problem was resolved in \cite{KMS}, in which the authors determined the optimal competitive ratio. Somewhat unfortunately, the optimal algorithm does not really use multiple robots simultaneously: all but one robot search on one ray each, while the last robot performs the search on all remaining rays. The time version of the problem was addressed in \cite{BFZ}, where it was resolved for {\it cyclic strategies}. Informally speaking, a cyclic strategy is a strategy in which the advancements in the search on the rays is happening in cyclic order, and at each step each robot is assigned a farther distance to explore on a ray than it previously explored on other rays.\vskip+0.1cm

We address the natural faulty generalization of the problem, when we have $m$ rays and $k$ robots, $f$ out of which are faulty. Let us denote the corresponding function by $A(m,k,f)$. That is, $A(k,f) = A(2,k,f)$. The following theorem gives the precise value of $A(m,k,f)$ in all (meaningful) cases. In particular, the $f=0$ case of our theorem answers the question posed in several research papers mentioned before. It strengthens the result from \cite{BFZ}, answering their question whether their result is possible to generalize to all strategies, gives appropriate analogue of the result from \cite{KMS} for the competitive ratios measured in terms of time, as well as answers the question posed by Baeza-Yates, Culberson, and Rawlins in \cite{BYC2}.

\begin{theorem}\label{thmmain2} Given that $f< k < m(f+1)$ and $q := m(f+1)$, we have
\begin{equation}\label{eqm} A(m,k,f)= \lambda_0:= 2\sqrt[k]{\frac{q^{q}}{(q-k)^{q-k}k^k}}+1.\end{equation}
\end{theorem}
Note that the restriction on $k$ simply excludes trivial cases: first, $f=k$ means that all robots are faulty. Second, if $k= m(f+1)$ (or larger), then sending $f+1$ robots on each of the $m$ rays will guarantee a competitive ratio of $1$.

It is easy to see that, substituting $m=2$ in \eqref{eqm}, one gets \eqref{eq01}. \vskip+0.1cm

The proof of the upper bound is deferred to the appendix. The proof of the lower bound uses the following covering relaxation of the original problem.
\vskip+0.1cm
{\sc One-ray cover with returns (ORC) setting: }  The goal is to cover (a subset of) $\RRone$. The robot (robots) starts at $0$ and move with unit speed along the ray $\RR_{\ge 0}$. One robot may cover a point multiple times, but different coverings are only counted if the robot visited $0$ in between.\vskip+0.1cm

Fix $k\in \mathbb N$ and $\lambda>1$. Then  a $q$-fold $\lambda$-cover of $\RRone$ in the ORC setting may be seen as a relaxation of the $m$-ray cover problem. Indeed, in a sense we simply discard the labels of the rays, but still ask the robots to return to $0$, imitating the change of the rays, which happened in the original problem. Thus, any strategy for searching a target on $m$ rays with $k$ robots, $f$ out of which are faulty, with competitive ratio $\lambda$, induces a strategy for the $q$-fold $\lambda$-covering of $\RRone$ with $k$ robots in the ORC setting, where $q=(f+1)m$.
\vskip+0.1cm

Let $C(k,q)$ be the infimum of all $\lambda$, for which there exists a strategy for a $q$-fold $\lambda$-covering of $\RRone$ with $k$ robots in the ORC setting. We prove the following:

\begin{equation}\label{eqmain}
  C(k,q)\ge 2\sqrt[k]{\frac{q^{q}}{(q-k)^{q-k}k^k}}+1.\end{equation}

Clearly, $A(m,k,f)\ge C(k,m(f+1))$, and thus the lower bound Theorem~\ref{thmmain2} follows from \eqref{eqmain}. At the same time, the bound \eqref{eqmain} is tight, as follows from the ``$\le$'' part of \eqref{eqm}.
Before proving \eqref{eqmain}, let us describe yet another covering problem, which is a fractional analogue of the multicovering problem in the ORC setting.\\

{\sc Fractional one-ray retrieval with returns: } We are given a finite number of robots, of total weight $1$ each of which moves with constant speed 1 along the ray $\RR_{\ge 0}$. We are supposed to cover an unknown target at distance $x\ge 1$ with several robots of some prescribed total weight $\eta,$ $\eta\in \RRone$. The same robot may cover any point any number of times, but different coverings are counted only if the robot visits $0$ in between. We are interested in the worst-case competitive ratio $C(\eta)$ of the best algorithm for the problem.\vskip+0.1cm


\textbf{Remark. } Here's a related, but not equivalent formulation, which is closer to the original $m$-ray problem. We are given a finite number of robots, each of which moves with constant speed 1 along the ray $\RR_{\ge 0}$ and is supposed to match an unknown target at distance $x\ge 1$ with one of the sample targets in a list of size $\eta,$ $\eta\in \RRone$. The total size of the memory (that is, the total portion of the list they may carry) of all robots is fixed and is equal to $1$. Robots may change the set of samples they carry only when visiting $0$. In the worst case, the target is matched only after being compared with all samples.  We are interested in the worst-case competitive ratio $C'(\eta)$ of the best algorithm for the problem.\\

The statement we prove is as follows:
\begin{equation}\label{eqmain3}
C(\eta) = 2\frac{\eta^{\eta}}{(\eta-1)^{\eta-1}}+1.\end{equation}
Naturally, the same equality holds for $C'(\eta)$ as well.
\vskip+0.1cm
We defer the proof of the reduction of \eqref{eqmain3} to \eqref{eqm} to the appendix.

\subsection{Proof of the lower bound in \eqref{eqmain}}

For technical reasons, we will prove the following slightly stronger statement about covering a finite part of $\RRone$: for any $\epsilon>0$, there exists $N$, such that if $k$ robots produce a $q$-fold $\lambda$-cover of $[1,N]$ in the ORC setting, then

\begin{equation}\label{eqmain2}
  \lambda \ge 2\sqrt[k]{\frac{q^{q}}{(q-k)^{q-k}k^k}}+1-\epsilon.\end{equation}
The main point here is that the needed $N$  is independent of the strategy.
Let us denote by $\mu(q,k)$ the root in the right hand side of the displayed inequality above. Note that $\mu(q,k)=\mu(cq,ck)$ for any $c>0$ and thus $\mu(q,k)<\mu(q-1,k-1)$, provided that $q>k>1$. Put $$\epsilon':=2\mu(q-1,k-1)-2\mu(q,k).$$ \vskip+0.1cm

 In the rest of the proof we show the validity of \eqref{eqmain2}.\vskip+0.1cm

 The proof of the theorem follows similar steps as the proof of Theorem~\ref{thm1}. We are presenting the proof in a somewhat contracted form, highlighting the changes one has to make in order to prove this theorem.\vskip+0.1cm

 The proof goes by induction on $k$.\footnote{We remark that induction is needed only for one part of the proof, which is given in Case~2 below.}   For $k=1$ we of course will not use any inductive hypothesis. For $k\ge 2$, suppose by induction that the statement of \eqref{eqmain2} is valid for $k-1$ and $q-1$. Choose $N'$, such that \eqref{eqmain2} holds for $k-1$ and $q-1$ with $N:=N'$ and $\epsilon:=\epsilon'$. In other words, for such choice of $N'$ any $(q-1)$-fold $\lambda'$-covering of $[1,N']$ by $k-1$ robots in the ORC setting satisfies \begin{equation}\label{eqind}\lambda'\ge 2\mu(q-1,k-1)+1-\epsilon' = 2\mu(q,k)+1.\end{equation}\vskip+0.1cm

Fix $\epsilon>0$ and a competitive ratio $\lambda$ not satisfying $\eqref{eqmain2}$ and put $\mu:=(\lambda-1)/2.$ Fix a collective strategy of robots for $k$ robots to produce a $q$-fold $\lambda$-covering of $[1,N]$ in the ORC setting (note that $N$ is sufficiently large and would be chosen later).

\paragraph{Standardising the strategy. }
The {\it round} of a strategy for a robot is the period between two consecutive visits of $0$. We can assume that in each round the robot turns exactly once at a turning point $t$. Therefore, we can describe a strategy by an infinite vector $T = (t_1,t_2,t_3, \ldots)$ of reals in $\RRone$, $t_i$ the turning point in round $i$. Obviously, point $x \in \RRone$ is $\lambda$-covered in round $i$ iff (i) $x \le t_i$ and (ii) $2(t_1+t_2+ \cdots + t_{i-1}) + x \le \lambda x$ ($\Leftrightarrow x \ge \nicefrac{1}{\mu}(t_1+t_2+ \cdots + t_{i-1})$ for $\mu:= \nicefrac{\lambda -1}{2}$). We set $t''_i:= \nicefrac{1}{\mu}(t_1+t_2+ \cdots + t_{i-1})$. If $t''_i > t_i$, round $i$ does not $\lambda$-cover any point, and we may as well skip this round (future rounds will $\lambda$-cover even more in this way). Otherwise, we call the round \emph{fruitful} and the robot $\lambda$-covers exactly the interval $[t''_i,t_i]$ in round $i$. Clearly, we can assume that we are using only strategies with only fruitful rounds. Observe that the sequence $(t''_1,t''_2,t''_3, \ldots)$ is monotone increasing.

Suppose now we have $k$ such strategies $T^{(r)} = (t^{(r)}_1,t^{(r)}_2,t^{(r)}_3, \ldots)$, $r=1,2,\ldots,k$, that establish a $q$-fold $\lambda$-covering of $\RRone$. Let us now assign truncated intervals $({t'_{i}}^{(r)},t_{i}^{(r)}] \subseteq [{t''_{i}}^{(r)},t_{i}^{(r)}] \cap [1,\infty)$, in such a way that every point in $\RROne$ is covered exactly $q$ times, and such that each sequence $({t'_1}^{(r)},{t'_2}^{(r)},{t'_3}^{(r)}, \ldots)$ is monotone increasing (this may actually also result in skipping some of the turning points).

\paragraph{Defining the function.}
Next we accumulate all assigned intervals of all $k$ robots in one sequence, sorted  by their left endpoints, ties broken arbitrarily. Consider a prefix ${\cal P}$ of this sequence, long enough so that all robots have already some interval present in ${\cal P}$. Clearly, it consists of all rounds of the strategy of the robot $r$ up to some $i_r$, for each $r=1,\ldots, k$. What part of $\RRone$ do these intervals cover, and how many times? There is a value $a = a({\cal P}) \ge 1$ such that $(1,a]$ is covered $q$ times, and then there are points $$a=:a_q\le a_{q-1}\le a_{s-2}\le\ldots\le a_1 ~,$$ such that every point in $(a_{j+1},a_j]$ is covered exactly $j$ times, $j=1,\ldots, q-1$, and $(a_1,\infty)$ is not yet covered. (Note that some of these intervals may be empty.) Associate with ${\cal P}$ the \emph{multiset} $$A({\cal P}) := \{ a_q, a_{q-1}, \ldots, a_1\}$$ (which can be seen as a description of the \emph{covering situation} of ${\cal P}$).

We need one more notion. Define the {\it load}, $L^{(r)}({\cal P})$, of robot $r$ in ${\cal P}$ as the sum of all turning points occurring in the intervals of $r$ in ${\cal P}$, i.e.\
$$
L^{(r)}({\cal P}) : = t_1^{(r)} + t_2^{(r)} + \cdots + t_{i_r}^{(r)} \mbox{~~ for $i_r$ s.t. $\left\{\begin{array}{ll} \mbox{$({t'_{i_r}}^{\!\!\!(r)},t_{i_r}^{(r)}]$ in ${\cal P}$, and} \\ \mbox{$({t'_{i_r+1}}^{\!\!\!\!\!\!\!\!\!\!(r)\,\,},t_{i_r+1}^{(r)}]$ not in ${\cal P}$.} \end{array}\right.$}
$$

Let $b^{(r)}:={t'}_{i_{r+1}}^{(r)}$ be the beginning of the first interval, assigned to robot $r$ and which is not in $\mathcal P$. Then, clearly, the analogue of \eqref{eq:LoadBound} is \begin{equation}\label{eqload} L^{(r)}(\mathcal P)\le \mu b^{(r)}.\end{equation}
The following function controls the situation:
\begin{equation}\label{eqfunc2}
f({\cal P}):=\prod_{r=1}^k \left[\frac{\left(L^{(r)}({\cal P})\right)^{q-k}(b^{(r)})^k}{\prod_{y \in A({\cal P})}y}\right] ~.
\end{equation}\vskip+0.1cm

\paragraph{Adding one interval.}
Add one new interval to $\mathcal P$, thus forming the covering situation $\mathcal P^+$. If the new interval is assigned to the robot $r$, then $b^{(r)}=a$. In $\cal P^+$, the beginning $b^{(r)}$ of the first non-included assigned interval to robot $r$ is replaced by some $b'$ (equal to ${t'}_{i_{r+2}}^{(r)}$); using \eqref{eqload}, the load of $r$ changes to $\mu^{*}b'$ for some $\mu^*\le \mu$; at the same time, the element $a$ is removed from $\mathcal A$ and replaced by the end of the new interval, which is $L^{(r)}(\mathcal P^+)-L^{(r)}(\mathcal P) = \mu^{*}b'-\mathcal L^{(r)}(\cal P)$.  Thus, in $f(\cal P)$ the $a^k$ multiple in the denominator is compensated by the $(b^{(r)})^k$ in the numerator, and we have
$$
\frac{f({\cal P}^+)}{f({\cal P})} = \frac{(\mu^* b')^{q-k}\cdot b'^k}{\left(L^{(r)}({\cal P})\right)^{q-k}\left(\mu^* b' - L^{(r)}({\cal P})\right)^k} = \frac{{\mu^*}^{q-k}}{x^{q-k} (\mu^* - x)^k}
$$
for $x := \frac{L^{(r)}({\cal P})}{b'}$, $0 < x < \mu^*$.

Making use of Lemma~\ref{lemgeom}, we conclude is as follows: if the competitive $\lambda$ does not satisfy \eqref{eqmain2} with $\epsilon$, then there exists $\delta>1$, such that \begin{equation}\label{eqgrow}\frac{f({\cal P}^+)}{f({\cal P})}\ge \delta.\end{equation}

The last step of the proof is to conclude that, at the same time, $f(\cal P)$ should be bounded (and that the number of intervals in the cover, and thus the number of such potential increments, is infinite). It is slightly more difficult in this case and is the reason why we use induction. We consider two cases.

\textbf{Case 1. } For any robot $r$  and any two consecutive starting points $t'^{(r)}_i$, $t'^{(r)}_{i+1}\in \RRone$ of its assigned intervals, we have $t'^{(r)}_{i+1}/t'^{(r)}_{i}\le C$, where $C$ is an absolute constant, which we would specify later.

In this case, looking at \eqref{eqfunc2}, we have $b^{(r)}\le Ct'^{(r)}_{i_r}\le Ca$, and thus $L^{(r)}(\mathcal P)\le \mu Ca$. This allows us to bound $f(\mathcal P)$:
$$f(\mathcal P)\le C^{qk} \mu^{(q-k)k}.$$
Therefore, in a finite number of steps, we will get a contradiction with \eqref{eqgrow}.
At the same time, when passing from $\mathcal P$ to $\mathcal P^+$, the distance from the first point that is not $q$-covered to zero increases by at most $C$ times. Thus, eventually, we get a contradiction. We can find the value of $N$ needed in \eqref{eqmain2} based on these two simple facts. \vskip+0.1cm

\textbf{Case 2. } There exists a robot $r$  and two consecutive starting points $t'^{(r)}_i$, $t'^{(r)}_{i+1}\in \RRone$ of its assigned intervals, such that $t'^{(r)}_{i+1}/t'^{(r)}_{i}\ge C$.  Consider the interval $$[\mu t'^{(r)}_{i},Ct'^{(r)}_{i}].$$ Then, due to \eqref{eqload}, no endpoint of the first $i-1$ intervals assigned to $r$ could have surpassed $\mu t'^{(r)}_{i}$. Thus, the interval displayed above is $\lambda$-covered by $r$ at most once. It implies that the other $k-1$ robots produced a $(q-1)$-fold $\lambda$-covering of this interval in the ORC setting. Rescaling, we may assume that $k-1$ robots produced a $(q-1)$-fold $\lambda$-covering of the interval $[1,C/\mu]$. Using \eqref{eqind} and choosing $C$ such that $C/\mu\ge N'$, we conclude that $\lambda\ge 2\mu(q,k)+1$.
\vskip+0.2cm

\paragraph{Acknowledgements.} We thank D\' aniel Kor\' andi, Alexander Pilz, Milo{\v s} Stojakovi{\' c}, and May Szedlak for discussions on and suggestions for the material covered in this paper, and Yoshio Okamoto for suggesting the problem.

\section{Appendix}

First, we give a proof of the upper bound in \eqref{eqmain}. We exhibit a so-called exponential strategy, similar to the many strategies that appeared for related problems. It is a fairly natural extension of the proofs of the upper bounds for some particular cases of our problem, exhibited in \cite{FR} and \cite{BFZ}.

\begin{proof}[\textsc{Proof of the upper bound in \eqref{eqmain}}]  It is sufficient, for each $x\in \RRone$, to produce an $f$-fold $\lambda_0$-covering each of points on the $m$ rays at distance $x$.

Let us first do the {\it assignment} for each of the robots: to each robot $r$ we assign intervals on different rays, which are supposed to be $\lambda_0$-covered by $r$. Fix $\alpha >1$, which we optimize later. The assignment is as follows: \begin{displayquote}robot $r$  $\lambda_0$-covers intervals $(\alpha^{k(i+mj)+m(r-f)},\alpha^{k(i+mj)+mr}]$ for each $j=-2,-1,\ldots$ on ray $i$, $i=1,\ldots, m$.\end{displayquote}
We start with $j=-2$ to ensure that before distance $1$ on each ray, each of the robots has already done at least $1$ pass on the ray. Let us first make sure that all the points at distance at least $1$ from the origin are covered at least $f$ times by such an assignment. Fix a ray $i$ and a point on it, in the form $\alpha^x$, where $x\ge 0$. Then $x$ must fall into  $(k(i+mj)+m(r-f),k(i+mj)+mr]$ for some $j$ in order to be covered by $r$. It is equivalent to $y\in (kj+r-f,kj+r]$, where $y:=\lceil\frac{x-ki}m\rceil$. Note that $y\ge -k$. Thus, $y$ falls into the interval above if $y(\!\!\!\mod k) \in (r-f(\!\!\!\mod k), r(\!\!\!\mod k)]$. This holds for $r\in [y(\!\!\!\mod k),y+f(\!\!\!\mod k)).$ Thus, the point $\alpha^x$ was covered by the intervals assigned by the robots $y-f+1,\ldots, y$ (modulo $k$).

Now let us describe a strategy for robots to cover these intervals, and calculate the corresponding competitive ratio. Each robot starts his tour with ray number $1$, and visits the rays in cyclic order. When visiting ray $i$ for the $(j+3)$-nd time ($+3$ is there since $j$ starts from $-2$), he turns at the point $\alpha^{k(i+mj)+mr}$, returns to 0 and continues to the ray $i+1$. A point at distance $x\ge 1$ from the origin on ray $i$ is covered for the $f$-th (and final) time by robot $r,$ if it falls into the segment $(\alpha^{k(i+mj)+m(r-f)},\alpha^{k(i+mj)+m(r-f+1)}]$ for some $j\in \mathbb Z$. Take $x\in (\alpha^{k(i+mj)+m(r-f)},\alpha^{k(i+mj)+m(r-f+1)}]$. Then the time passed until robot $r$ arrives at this point is at most
$$x+2\sum_{l=-2m}^{i+mj-1}\alpha^{kl+mr}< x+2\frac{\alpha^{k(i+mj)+mr}}{\alpha^k-1}.$$
Thus, the competitive ratio of such algorithm is $2\gamma +1$, where $\gamma$ is the maximum over all valid choices of $x$ of the expression $$\frac{\alpha^{k(i+mj)+mr}}{x(\alpha^k-1)}.$$
This expression is clearly decreasing as $x$ increases, therefore,
$$\gamma\le\frac{\alpha^{k(i+mj)+mr}}{\alpha^{k(i+mj)+m(r-f)}(\alpha^k-1)} = \frac{\alpha^{mf}}{\alpha^k-1}.$$
Note that the last bound is independent of the choice of the ray and $x$.  Applying Lemmas~\ref{lemmax},~\ref{lemgeom} with $\mu^{*}=1$, $x =\alpha^{-k}$, $s=\frac{mf-k}k$, we get that the displayed expression is minimised for $\alpha = \sqrt[k]{\frac{mf}{mf-k}}$, for which it is equal to $(\lambda_0-1)/2$. We conclude that the competitive value of the algorithm coincides with the right hand side of  \eqref{eqm}.\end{proof}
\vskip+0.2cm

Next, we give the reduction of \eqref{eqmain3} to \eqref{eqmain}.
\begin{proof}[\textsc{Proof of \eqref{eqmain3}}]
Let us first prove the ``$\le$'' direction: that the right hand side is an upper bound for the left hand side. Fix a sequence of rational numbers $\frac{q_i}{k_i}, i=1,\ldots$, such that $\frac{q_i}{k_i}\ge \eta$ and $\lim_{i\to \infty}\frac{q_i}{k_i}= \eta$. Consider the strategy that gives the upper bound from Theorem~\ref{thmmain2} with $f:=1$, $k:=k_i, m:=q_i$. Then the competitive ratio is
$$2\sqrt[k_i]{\frac{q_i^{q_i}}{(q_i-k_i)^{q_i-k_i}k_i^{k_i}}}+1 = 2\frac{(q_i/k_i)^{q_i/k_i}}{(q_i/k_i-1)^{(q_i/k_i-1)}}+1\to 2\frac{\eta^{\eta}}{(\eta-1)^{\eta-1}}+1.$$
On the other hand, the strategy above can be clearly used to produce a fractional one-ray $\frac{q_i}{k_i}$-covering (and thus an $\eta$-covering): just split the weight between $k_i$ robots in equal parts, and let them perform their strategy on one ray, ignoring the labels of the rays. Going to the limit, the ``$\le$'' part of \eqref{eqmain3} is proved.\\

The proof of the ``$\ge$'' direction is similar. Fix $\epsilon>0$ and consider a strategy $S$ in the fractional problem, which achieves the competitive ratio $C(\eta)+\epsilon$. Assume that the weights of the robots used in the strategy are $w_1,\ldots, w_n$, $\sum_{i=1}^n w_i=1$. Fix $\delta>0$, which choice is specified later, and integers $q,k_1,\ldots, k_n$, such that $\frac{w_i}{\eta}\le \frac{k_i}q\le \frac{w_i}{\eta}+\delta$. Put $k:=\sum_{i=1}^n k_i$. Thus, clearly, $\frac kq\ge \frac 1{\eta}$. The requirement on $\delta$ is such that $\frac{k}{q}\le \frac{1}{\eta-\epsilon}$.

Consider the strategy $S'$ for the $q$-fold covering with $k$ robots in the ORC setting, which is obtained as follows: the first $k_1$ robots repeat the actions of the first robot from strategy $S$, the next $k_2$ robots repeat the actions of the second robot from $S$ etc. Clearly, the strategy produces the $q$-fold covering and has competitive ratio $C(\eta)+\epsilon$. On the other hand, using the lower bound from \eqref{eqmain}, it has competitive ratio at least
$$2\frac{(q/k)^{q/k}}{(q/k-1)^{(q/k-1)}}+1\ge 2\frac{(\eta-\epsilon)^{\eta-\epsilon}}{(\eta-\epsilon-1)^{\eta-\epsilon-1}}+1.$$
This implies that for any $\epsilon>0$ we have
$$C(\eta)\ge 2\frac{(\eta-\epsilon)^{\eta-\epsilon}}{(\eta-\epsilon-1)^{\eta-\epsilon-1}}+1-\epsilon.$$
Passing to the limit $\epsilon\to 0$ gives the result.
\end{proof}

\end{document}